\documentclass[12pt]{alt2020} 

\usepackage{amsmath}
\usepackage{amsfonts}

\def\R{{\mathbb{R}}}
\def\pr{{\rm Pr}}

\def\H{{\mathcal H}}

\def\D{{\mathcal D}}
\def\dG{\overleftrightarrow{G}}
\def\uG{\overline{G}}

\newtheorem{thm}{Theorem}

\newtheorem{cor}[thm]{Corollary}

\newtheorem{defn}[thm]{Definition}

\newtheorem{open}{Open problem}

\title[What relations are reliably embeddable in Euclidean space?]{What relations are reliably embeddable in Euclidean space?}
\usepackage{times}

\altauthor{%
 \Name{Robi Bhattacharjee} \Email{rcbhatta@eng.ucsd.edu}\\
 \addr Department of Computer Science\\
       University of California, San Diego\\
       San Diego, CA 92093, USA
 \AND
 \Name{Sanjoy Dasgupta} \Email{dasgupta@eng.ucsd.edu}\\
 \addr Department of Computer Science\\
       University of California, San Diego\\
       San Diego, CA 92093, USA
}

\begin{document}

\maketitle

\begin{abstract}%
  We consider the problem of embedding a relation, represented as a directed graph, into Euclidean space. For three types of embeddings motivated by the recent literature on knowledge graphs, we obtain characterizations of which relations they are able to capture, as well as bounds on the minimal dimensionality and precision needed.
\end{abstract}


\section{Introduction}

The problem of \textit{embedding graphs in Euclidean space} has arisen in a variety of contexts over the past few decades. Most recently, it has been used for making symbolic knowledge available to neural nets, to help with basic reasoning tasks~\citep{NMTG16}. This knowledge consists of relations expressed in tuples, like ({\tt Tokyo}, {\tt is-capital-of}, {\tt Japan}). Alternatively, each relation (like {\tt is-capital-of}) can be thought of as a directed graph whose nodes correspond to {\it entities} (such as cities and countries).

A wide array of methods have been proposed for embedding such relations in vector spaces~\citep{PH01,KTGYU06,STS09,BWCB11,NTK11,AUGWY13,SCMN13,NK17}. For instance, {\it translational embeddings}~\citep{AUGWY13} map each entity $x$ to a vector $\phi(x) \in \R^d$ and each relation $r$ to a vector $\Psi(r) \in \R^d$. The intention is that for any entities $x,y$ and any relation $r$,
$$ \mbox{relation $(x,r,y)$ holds} \ \ \Longleftrightarrow \ \ \phi(x) + \Psi(r) \approx \phi(y).$$
This is motivated in part by the success of {\it word embeddings}~\citep{MSCCD13}, which embed words in Euclidean space so that words with similar co-occurrence statistics lie close to one another. It has been observed that these embeddings happen to obey linear relationships of the type above for certain relations and entities, making it possible, for instance, to use them for simple analogical reasoning~\citep{MYZ13}. Rather than relying upon these haphazard coincidences, it makes sense to explicitly embed relations of interest so that this property is assured.

An alternative scheme, {\it structured embeddings}~\citep{BWCB11}, again assigns each entity $x$ a vector $\phi(x) \in \R^d$, but assigns each relation $r$ a pair of $d \times d$ matrices, $L_r$ and $R_r$, so that
$$ \mbox{relation $(x,r,y)$ holds} \ \ \Longleftrightarrow \ \ L_r \phi(x) \approx R_r \phi(y).$$
Notice that this is more general than translational embeddings because $L_r$ can capture any affine transformation by adding a constant-valued feature to $\phi$.

Another example of an embedding method is {\it bilinear embedding}~\citep{NTK11}, in which each entity $x$ gets a vector $\phi(x) \in \R^d$ and each relation $r$ gets a matrix $A_r$, so that
$$ \mbox{relation $(x,r,y)$ holds} \ \ \Longleftrightarrow \ \ \phi(x)^TA_r \; \phi(y) \ \geq \ \mbox{some threshold}.$$

These three embedding methods---translational, structured, and bilinear---broadly represent the various schemes that have been proposed in the recent machine learning literature, and many other suggestions are variants of these. For instance, {\it linear relational embedding}~\citep{PH01} assigns each entity $x$ a vector $\phi(x)$ and each relation $r$ a matrix $M_r$ so that $(x,r,y) \Longleftrightarrow \phi(y) \approx M_r \phi(x)$: a special case of structured embedding.

Typically the parameters of the embeddings (the mapping $\phi$ as well as the vectors and matrices for each relation) are fit to a given list of relation triples, using some suitable loss function. They can then be used for simple reasoning tasks, such as link prediction.

In this paper, we take a formal approach to this whole enterprise.
\begin{enumerate}
\item What kinds of relations can be embedded using these methods? Can arbitrary relations be accurately represented?
\item What dimensionality is needed for these embeddings?
\item What precision is needed for these embeddings? This question turns out to play a central role.
\end{enumerate}
In particular, we will think of a relation as being reliably embeddable if it admits an embedding that does not require too much precision or too high a dimension. We wish to gauge what kinds of relations have this property.

In order to answer these questions, it is enough to look at a single relation at a time. We therefore look at the problem of {\it embedding a given directed graph in Euclidean space}.

\subsection{Related work}

There is a substantial literature on embedding {\it undirected} graphs into Euclidean space. A key result is the following.
\begin{thm}[\cite{M84}]
For any undirected graph $G = (V,E)$, there is a mapping $\phi: V \rightarrow \R^d$ such that $\{u,v\} \in E \Longleftrightarrow \|\phi(u) - \phi(v)\| \leq 1$. Here $d \leq |V|$.
\label{thm:sphericity}
\end{thm}
We will call this an \textit{undirected distance embedding} to avoid confusion with embeddings of {\it directed} graphs, our main focus. The {\it sphericity} of an undirected graph $G$ is defined as the smallest dimension $d$ for which such an embedding exists; computing it was proved NP-hard by \citet{KM12}. The same paper showed an even more troubling result, that embeddings achieving this minimum dimension sometimes require precision (number of bits per coordinate) {\it exponential} in $|V|$. This has been a key consideration in our formulation of robustness.

An embedding of an undirected graph can also be based on dot products rather than Euclidean distance. We call these \textit{undirected similarity embeddings}. The following is known.
\begin{thm}[\cite{RRS89}]
For any undirected graph $G = (V,E)$, there is a mapping $\phi: V \rightarrow \R^d$ such that $\{u,v\} \in E \Longleftrightarrow \phi(u) \cdot \phi(v) \geq 1$. Here $d \leq |V|$.
\label{thm:sphericity-dot-product}
\end{thm}
The minimum dimension $d$ needed for an undirected similarity embedding is at most the sphericity of $G$, but can be much smaller. A complete binary tree on $n$ nodes, for instance, has sphericity $\Omega(\log n)$ but can be embedded in $\R^3$ using a dot-product embedding~\citep{RRS89}.

The present paper is about embeddings of {\it directed} graphs, and many of the results we obtain are qualitatively different from the undirected case. We also diverge from earlier theory work by giving {\it precision} a central role in the analysis, via a suitable notion of robustness.

Another body of work, popular in theoretical computer science, has looked at {\it embeddings of distance metrics} into Euclidean space~\citep{LLR95}. Here a metric on finitely many points is specified by an undirected graph with edge lengths, where the distance between two nodes is the length of the shortest path between them. The idea is to find an embedding of the nodes into Euclidean space that preserves all distances. For many graphs, an embedding of this type is not possible: for constant-degree expander graphs, for instance, a multiplicative distortion of $\Omega(\log n)$ is inevitable, where $n$ is the number of nodes. The problem we are considering differs in two critical respects: first, we only need to preserve immediate neighborhoods, and second, we are dealing with directed graphs.

The machine learning literature has proposed many methods for embedding, such as those mentioned above, along with empirical evaluation. There has also been work on embeddings into non-Euclidean spaces: complex-valued~\citep{TWRGB16} or hyperbolic~\citep{NK17}. In this paper, we focus on the Euclidean case.

\subsection{Overview of results}

Let $G = (V,E)$ be a directed graph representing a relation we wish to embed. Here $V$ is the set of entities, and an edge $(u,v)$ means that the relation holds from $u$ to $v$.

We begin by considering a formalization of {\it translational embeddings}. We find  that only a limited class of relations can be embedded this way: a directed cycle does not have a translational embedding (Theorem~\ref{thm:translational-negative}), but any directed acyclic graph does (Theorem~\ref{thm:translational-positive}).

Next, we consider more powerful classes of embeddings: abstractions of the structured and bilinear embeddings mentioned above, that we call {\it distance embeddings} and {\it similarity embeddings}, respectively. We find, first, that all directed graphs admit both types of embeddings (Theorem~\ref{thm:distance-embedding-existence}). Moreover, the minimum dimension achievable in the two cases differs by at most 1 (Theorem~\ref{thm:relating-dimensions}), and is closely related to the sign rank of the adjacency matrix of the graph (Theorem~\ref{thm:sign-rank}). We present several examples of embeddings for canonical types of graphs: paths, cycles, trees, and so on.

We also explicitly focus on the precision of embeddings, which has not been a feature of the earlier theory work on undirected graphs. In particular, we introduce a notion of $\delta$-robustness, where larger values of $\delta$ correspond to more robust embeddings. We relate this directly to precision by showing that any graph that admits a $\delta$-robust embedding also has distance and similarity embeddings into the $O((1/\delta^2) \log n)$-dimensional Boolean hypercube (Theorem~\ref{thm:hamming}). In this way, the $\delta$ parameter translates directly into an upper bound on the number of bits needed. We look at the robustness achievable on different families of graphs. We find, for instance, that for any graph of maximum degree $D$, robustness $\delta \geq 1/D$ can be attained (Theorem~\ref{thm:distance-embedding-existence}). On the other hand, random dense graphs are not robustly embeddable (Corollary~\ref{cor:random-graphs}).

Our analysis of embeddings focuses on two parameters: dimension and robustness. We show that the former is NP-hard to minimize (Appendix B), while the latter can be maximized efficiently by semidefinite programming (Section~\ref{sec:algorithms}). Thus robustness is a promising optimization criterion for designing embeddings.

\section{Translational embeddings}

\begin{defn}
A \emph{translational embedding} of a directed graph $G = (V,E)$ is given by a mapping $\phi: V \rightarrow \R^d$, a unit vector $z \in \R^d$, and thresholds $\{t_u \geq 0: u \in V\}$, such that for all $u \neq v$,
$$ (u,v) \in E \ \Longleftrightarrow \ \|\phi(v) - (\phi(u) + z)\| \leq t_u .$$
If all the thresholds $t_u$ are identical, then we call it a \emph{uniform} translational embedding.
\end{defn}
Note that (i) the requirement that $z$ be a unit vector is without loss of generality, and (ii) we avoid checking self-edges in order to sidestep various complications. Paraphrasing, the definition imposes an ordinal constraint: if $(u,v) \in E$ but $(u,w) \not\in E$, then $\phi(v)$ must lie closer to $\phi(u) + z$ than does $\phi(w)$.

For instance, let $P_n$ denote the directed path $1 \rightarrow 2 \rightarrow \cdots \rightarrow n$. A uniform translational embedding in $\R$ is given by $\phi(k) = k$, $z = 1$, and any $0 < t < 1$. 

As another example, consider the directed complete bipartite graph containing all edges from node set $V_1$ to complementary node set $V_2$. A uniform translational embedding to $\R$ is again available: map
$$ \phi(u) 
=
\left\{
\begin{array}{ll}
0 & \mbox{for $u \in V_1$} \\
1 & \mbox{for $u \in V_2$}
\end{array}
\right.
$$
with $z = 1$ and any $0 < t < 1$.

It is of interest to determine what kinds of graphs can be embedded translationally. We begin with a negative result. 
\begin{thm}
$C_n$, the directed cycle on $n$ nodes, does not admit a translational embedding for any $n \geq 3$.
\label{thm:translational-negative}
\end{thm}

\begin{proof}
Assume that $u_1 \rightarrow u_2 \rightarrow \cdots \rightarrow u_n \rightarrow u_1$ has a translational embedding $(\phi, z, \{t_u\})$; we will arrive at a contradiction.

First, for any edge $x \rightarrow y$, the conditions for $(x,y) \in E$ and $(y,x) \not\in E$ are, respectively,
\begin{align*}
\| \phi(y) - (\phi(x) + z)\|^2 &\leq t_x^2 \\
\| \phi(x) - (\phi(y) + z)\|^2 &> t_y^2
\end{align*}
which can be rewritten
\begin{align*}
\| (\phi(y) - \phi(x)) - z\|^2 &\leq t_x^2 \\
\| (\phi(y) - \phi(x)) + z\|^2 &> t_y^2 .
\end{align*}
The left-hand sides have $\|\phi(y) - \phi(x)\|^2$ and $\|z\|^2$ in common. Subtracting, we get
$$ z \cdot (\phi(y) - \phi(x)) > \frac{1}{4} (t_y^2 - t_x^2) .$$

Now we can apply this to the $n$ edges of the cycle to yield the system of inequalities
\begin{align*}
z \cdot (\phi(u_2) - \phi(u_1)) &> \frac{1}{4} (t_{u_2}^2 - t_{u_1}^2) \\
z \cdot (\phi(u_3) - \phi(u_2)) &> \frac{1}{4} (t_{u_3}^2 - t_{u_2}^2) \\
&\vdots \\
z \cdot (\phi(u_1) - \phi(u_n)) &> \frac{1}{4} (t_{u_1}^2 - t_{u_n}^2) 
\end{align*}
The left-hand sides add up to zero, as do the right-hand sides, a contradiction.
\end{proof}

On the other hand, any directed {\it acyclic} graph can be translationally embedded.
\begin{thm}
Suppose directed graph $G = (V,E)$ is acyclic. Then $G$ admits a uniform translational embedding.
\label{thm:translational-positive}
\end{thm}

\begin{proof}
By topologically ordering $G$, assume without loss of generality that $V = \{1,2,\ldots, n\}$ and that all edges $(i,j) \in E$ have $i < j$. Let $G' = (V,E')$ denote the undirected version of $G$, with an edge $\{i,j\} \in E'$ for every $(i,j) \in E$. By applying a result of \citet{FM88}, we obtain an embedding $
\psi: V \rightarrow \R^d$ of $G'$ with the following characteristics:
\begin{itemize}
    \item $\|\psi(i)\|^2 = \Delta$, where $\Delta \geq 1$ is at most the maximum degree of $G'$.
    \item If $\{i,j\} \in E'$ then $\|\psi(i) - \psi(j)\|^2 = 2(\Delta-1)$.
    \item If $\{i,j\} \not\in E'$ then $\|\psi(i) - \psi(j)\|^2 = 2 \Delta$.
\end{itemize}

We then define a uniform translational embedding of $G$ into $\R^{d+1}$ as follows:
$$ \phi(i) = (i \delta, \psi(i)) ,$$
where $\delta = 1/(n-1)$. Take $z = e_1$, the first coordinate direction, and threshold $t = \sqrt{2\Delta - 1}$.

To see that this works, pick any $i < j$. First off, if $(i,j) \in E$, then $\{i,j\} \in E'$ and
\begin{align*}
    \|\phi(j) - (\phi(i) + z)\|^2 
    &= (j\delta - i\delta - 1)^2 + \|\psi(i) - \psi(j)\|^2 \\
    &\leq (1-\delta)^2 + 2(\Delta - 1) \ \leq \ t^2 .
\end{align*}
On the other hand, if $(i,j) \not\in E$, then $\{i,j\}\not\in E'$, and we have
$$
\|\phi(j) - (\phi(i) + z)\|^2 
\ > \ \|\psi(i) - \psi(j)\|^2 
\ = \ 2 \Delta \ > \ t^2 .
$$
Finally, we confirm that the embedding does not suggest a back edge from $j$ to $i$:
\begin{align*}
    \|\phi(i) - (\phi(j) + z)\|^2 
    &= (i\delta - j\delta - 1)^2 + \|\psi(i) - \psi(j)\|^2 \\
    &\geq (1+\delta)^2 + 2(\Delta - 1) \ > \ t^2 .
\end{align*}   
\end{proof}

\begin{open}
What characterization can be given for the \emph{minimum} dimension of a translational embedding of a dag? 
\end{open}

\section{Distance embeddings}

\begin{defn}
A \emph{distance embedding} of a directed graph $G = (V,E)$ is given by a pair of mappings $\phi_{in}, \phi_{out}: V \rightarrow \R^d$, and a threshold $t$, such that for all pairs of nodes $u, v$,
$$ (u,v) \in E \ \Longleftrightarrow \ \|\phi_{out}(u) - \phi_{in}(v) \| \leq t .$$
We will sometimes be interested in distance embeddings \emph{into the unit sphere}, where all $\phi_{in}(u)$ and $\phi_{out}(v)$ have length one.
\end{defn}

This formalism captures several types of embedding that have been proposed in the machine learning literature. Recall, for instance, the notion of a {\it structured embedding}~\citep{BWCB11}, given by $\phi: V \rightarrow \R^d$ and $d \times d$ matrices $L$ and $R$, where
$ (u,v) \in E \Longleftrightarrow L \phi(u) \approx R \phi(v).$
This can be converted into a distance embedding by taking $\phi_{out}(u) = L \phi(u)$ and $\phi_{in}(u) = R \phi(u)$. Conversely, if a graph has distance embedding $\phi_{in}, \phi_{out}: V \rightarrow \R^d$, then it has a structured embedding $(\phi: V \rightarrow \R^{2d}, L, R)$, where $\phi(u)$ is the concatenation of $\phi_{in}(u)$ and $\phi_{out}(u)$ and matrices $L$ and $R$ retrieve the bottom and top $d$ coordinates, respectively, of a $2d$-dimensional vector.

In the above formulation of distance embedding, there is a single threshold, $t$, that applies for all points. An alternative would be to allow a different threshold $t_u$ for each node $u$, so that
$$ (u,v) \in E \ \Longleftrightarrow \ \|\phi_{out}(u) - \phi_{in}(v) \| \leq t_u .$$
This is easily simulated under our current definition, by adding an extra dimension. Given an embedding $\phi_{in}, \phi_{out}: V \rightarrow \R^d$ with varying thresholds $t_u$, we can define $\tilde{\phi}_{in}: \tilde{\phi}_{out}: V \rightarrow \R^{d+1}$ by $\tilde{\phi}_{in}(u) = (\phi_{in}(u), 0)$ and $\tilde{\phi}_{out}(u) = (\phi_{out}(u), \sqrt{t^2 - t_u^2})$, where $t = \max_u t_u$. Then
$$ \|\phi_{out}(u) - \phi_{in}(v) \| \leq t_u \ \Longleftrightarrow \ \|\tilde{\phi}_{out}(u) - \tilde{\phi}_{in}(v)\| \leq t .$$

We will shortly see that every directed graph has a distance embedding. It is of interest, then, to characterize the minimum achievable dimension.
\begin{defn}
Let $d_{dist}(G)$ be the smallest dimension $d$ of any distance embedding of $G$. Let $d^{\circ}_{dist}(G)$ be the smallest dimension of any distance embedding into the unit sphere.
\end{defn}

A useful observation is that $d_{dist}$ and $d_{dist}^\circ$ do not differ by much.
\begin{thm}
For any directed graph $G$, we have $d_{dist}(G) \leq d_{dist}^{\circ}(G) \leq d_{dist}(G) + 1$.
\label{thm:distance-dimension-spherical}
\end{thm}

\begin{proof}
The first inequality is trivial. We give an informal sketch of the second, since the details also appear in Theorem~\ref{thm:delta-dist-sim}. A distance embedding $\phi$ of $G$ in $\R^d$ can be mapped to an embedding $\phi'$ in a small neighborhood of the unit sphere $S^d \subset \R^{d+1}$. To see this, notice that scaling down $\phi$ (and $t$) by a constant factor maintains the embedding property. If they are sufficiently downscaled that the set of embedded points lies within a $d$-dimensional ball of very small radius, then this ball can be placed close to the surface of the unit sphere in $\R^{d+1}$, and the points can be projected to the surface of the sphere while inducing an arbitrarily small multiplicative distortion in pairwise distances.
\end{proof}

As described in the introduction, earlier work has brought out troubling pathologies in the precision required for embedding an undirected graph: achieving the minimum possible dimension could require the vectors to be specified using a number of bits that is exponential in $|V|$~\citep{KM12}. For this reason, we keep careful track of precision. Our key tool in doing so is a notion of {\it robustness}, which we will later relate to both precision and dimension.
\begin{defn}
Suppose a distance embedding of a directed graph $G = (V,E)$ is given by $(\phi_{in}, \phi_{out}, t)$. We say the embedding is $\delta$-robust, for $\delta > 0$, if
\begin{itemize}
\item $(u,v) \in E \ \Longrightarrow \ \|\phi_{out}(u) - \phi_{in}(v)\|^2 \leq t^2$.
\item $(u,v) \not\in E \ \Longrightarrow \ \|\phi_{out}(u) - \phi_{in}(v)\|^2 \geq t^2 (1+\delta)$.
\end{itemize}
\end{defn}

We now show that all directed graphs have distance embeddings.
\begin{thm}
Let $G = (V,E)$ be any directed graph. Let $A$ be its $|V|\times |V|$ adjacency matrix: that is, $A_{uv}$ is $1$ if $(u,v) \in E$ and $0$ otherwise. Let $k$ denote the rank of $A$ and $\sigma_1$ its largest singular value. Then $G$ has a distance embedding into the unit sphere in $\R^k$ that is $(1/\sigma_1)$-robust.
\label{thm:distance-embedding-existence}
\end{thm}

\begin{proof}
For convenience, label the vertices $1, 2, \ldots, n$. Take the singular value decomposition of $A$ so that $A = U^T\Sigma V$ where $U$ and $V$ are $n \times n$ orthogonal matrices, and $\Sigma$ is a diagonal matrix with entries $\sigma_1 \geq \sigma_2 \geq \dots \geq \sigma_n$. If rank $k < n$, then $\sigma_{k+1} = \cdots = \sigma_n = 0$.

Writing $A = (\Sigma^{1/2} U)^T (\Sigma^{1/2} V)$, take $\phi_{out}(i) \in \R^k$ to be the first $k$ coordinates of the $i$th column of $\Sigma^{1/2} U$ (the remaining coordinates are zero), and $\phi_{in}(i) \in \R^k$ to be the first $k$ coordinates of the $i$th column of $\Sigma^{1/2} V$. Then $A_{ij} = \phi_{out}(i) \cdot \phi_{in}(j)$. These vectors all have length at most $\sqrt{\sigma_1}$; normalize them to unit length, to get $\widehat{\phi}_{out}, \widehat{\phi}_{in}: V \rightarrow S^{k-1}$. Then
\begin{itemize}
    \item $(i,j) \in E \Longrightarrow \widehat{\phi}_{out}(i) \cdot \widehat{\phi}_{in}(j) \geq 1/\sigma_1$ and $\|\widehat{\phi}_{out}(i) - \widehat{\phi}_{in}(j)\|^2 \leq 2(1 - 1/\sigma_1)$.
    \item $(i,j) \not\in E \Longrightarrow \widehat{\phi}_{out}(i) \cdot \widehat{\phi}_{in}(j) = 0$ and $\|\widehat{\phi}_{out}(i) - \widehat{\phi}_{in}(j)\|^2 = 2$.
\end{itemize}
Setting $t = \sqrt{2(1-1/\sigma_1)}$, we see the embedding is $\delta$-robust for $\delta \geq 1/(1-1/\sigma_1) - 1 \geq 1/\sigma_1$.
\end{proof}

As a consequence, any graph of constant degree is robustly embeddable. The proof of the following corollary is deferred to the appendix.
\begin{cor}
Suppose all nodes in $G$ have indegree $\leq \Delta_-$ and outdegree $\leq \Delta_+$. Then $G$ has a distance embedding that is $\sqrt{1/(\Delta_+\Delta_-)}$-robust.
\label{cor:robustness-degree}
\end{cor}

\subsection{Differences from undirected embeddings}

At a first glance, it may seem that directed embeddings may not be significantly different from undirected embeddings considering standard transformations between the two types of graphs, as defined below. However, we will see that this is not the case, by considering some examples in which a directed graph can be embedded in much lower dimension than its undirected counterpart.
\begin{defn}
\begin{enumerate}
	\item For an undirected graph $G$, let $\dG$ be the directed graph which has edges $(u,v), (v,u) \in E(\dG)$ for every $\{u,v\} \in E(G)$.
	\item For a directed graph $G$, let $\uG$ be the undirected graph with 2 vertices $v_{out}, v_{in}$ for every $v \in V(G)$ and with edge $\{u_{out}, v_{in}\} \in E(\uG)$ for every $(u,v) \in E(G)$.
\end{enumerate}
\end{defn}

\begin{thm}\label{undir_diff_thm}
Let $G = K_{n,n}$ be the undirected complete bipartite graph. The sphericity of $G$ is $\Omega(n)$ whereas $d_{dist}(\dG)$ is $1$.
\end{thm}

A proof can be found in the appendix. An important intuition from this example is that embedding undirected bipartite graphs can be difficult because they have large independent sets with many common neighbors. However, for directed graphs this does not present a problem because of the flexibility that comes from having two embeddings, $\phi_{in}$ and $\phi_{out}$.

This idea is also what makes embedding $\uG$ significantly more difficult than embedding a directed graph $G$ as $\uG$ has large independent sets. 

\begin{thm}
Let $G = \overleftrightarrow{K_n}$ be a directed graph with every possible edge (including self loops). Then $d_{dist}(G) = 0$ while $\uG$ has sphericity $\Omega(n)$.
\end{thm}

\subsection{Robustness yields low dimensionality}

We now show that any graph with a $\delta$-robust embedding can be embedded in dimension $O((1/\delta^2) \log n)$.
\begin{thm}\label{thm_JL_project}
If $G$ has a $\delta$-robust distance embedding (in any dimension), then it also has a $\frac{\delta}{2}$-robust embedding in $O(\frac{1}{\delta^2}\log n)$ dimensions. 
\end{thm} 

\begin{proof}
This is a consequence of a lemma of \citet{JL84}. Let $\phi_{out}, \phi_{in}: V \rightarrow \R^d$, with threshold $t$, be a $\delta$-robust embedding of $G$. The JL lemma states that for any $\epsilon > 0$, there exists a map $f: \R^d \rightarrow \R^m$, with $m = O((\log n)/\epsilon^2)$, so that
$$(1-\epsilon)\|\phi_{out}(u) - \phi_{in}(v)\|^2 \leq \|f(\phi_{out}(u)) - f(\phi_{in}(v))\|^2 \leq(1+\epsilon) \|\phi_{out}(u) - \phi_{in}(v)\|^2,$$ 
for all $u, v \in V$. To ensure that the new embedding is $(\delta/2)$-robust, it suffices to take $\epsilon = \delta/8$.
\end{proof}

Later, we will see that a graph with a $\delta$-robust embedding is in fact robustly embeddable in the $O((1/\delta^2) \log n)$-dimensional {\it Hamming cube}. In this way, robustness implies the existence of a low-dimensional embedding that requires only one bit of precision per coordinate.
  
\section{Similarity embeddings}

\begin{defn}
A {\it similarity embedding} of a directed graph $G = (V,E)$ is given by a pair of mappings $\phi_L, \phi_R: V \rightarrow \R^d$ and a threshold $t$, such that
$$ (u,v) \in E \ \Longleftrightarrow \ \phi_L(u) \cdot \phi_R(v) \geq t .$$
We will often be interested in embeddings \emph{into the unit sphere}, where the $\phi_L(u)$ and $\phi_R(u)$ have unit norm. We use $L, R$ notation as opposed to $\{\mbox{\rm in}, \mbox{\rm out}\}$ to help distinguish similarity embeddings from distance embeddings.
\end{defn}

This is closely related to the notion of bilinear embedding~\citep{NTK11}, which assigns each node $u$ to a vector $\phi(u) \in \R^d$ so that
$ (u,v) \in E \Longleftrightarrow \phi(u)^T A \phi(v) \geq t $,
for some $d \times d$ matrix $A$. To obtain a similarity embedding, take $\phi_L(u) = \phi(u)$ and $\phi_R(u) = A \phi(u)$. Conversely, given a similarity embedding $\phi_L, \phi_R: V \rightarrow \R^d$, we can construct a bilinear embedding by setting $\phi(u)$ to the $2d$-dimensional concatenation of $\phi_L(u)$ and $\phi_R(u)$, and taking $A$ to be $\begin{pmatrix} 0 & I \\ 0 & 0 \end{pmatrix}$.

The distance embedding constructed in Theorem~\ref{thm:distance-embedding-existence} also functions as a similarity embedding. Thus, such embeddings exist for every graph.
\begin{defn}
For directed graph $G$, let $d_{sim}(G)$ denote the smallest dimension into which a similarity embedding can be given.
\end{defn}

We now see that the dimensions $d_{dist}, d_{dist}^{\circ}, d_{sim}$ are almost identical.
\begin{thm}
$d_{dist}^{\circ}(G) -1 \leq d_{sim}(G) \leq d_{dist}^{\circ}(G) \leq d_{dist}(G) + 1$ for any directed graph $G$.
\label{thm:relating-dimensions}
\end{thm}
\begin{proof}
The inequality $d_{sim}(G) \leq d_{dist}^{\circ}(G)$ is immediate: any distance embedding into the unit sphere automatically meets the requirements of a similarity embedding. The final inequality is from Theorem~\ref{thm:distance-dimension-spherical}. It thus remains to show that $d_{dist}^{\circ}(G) \leq d_{sim}(G) + 1$.

Let $\phi_L, \phi_R: V \rightarrow \R^d$ be a similarity embedding of $G$ with threshold $t$. Indexing vertices as $1, 2, \ldots, n$, let $M$ be an $n \times n$ matrix with $M_{ij} = \phi_L(i) \cdot \phi_R (j)$. 

If $J$ is the all-ones matrix, then $M - tJ$ is matrix of rank at most $d+1$ such that 
$$(M- tJ)_{ij} \geq 0 \Longleftrightarrow (i,j) \in E.$$  
We will extract a distance embedding into the unit sphere from this matrix. 

Express $M -tJ$ as $U^TW$ for $U, W \in \R^{(d+1) \times n}$. Next, normalize the columns of $U$ and $W$ to unit length, to get $\widehat{U}$ and $\widehat{W}$. The key idea is that the pairwise dot products between these unit vectors still satisfy the above criterion. In short, taking $\phi_{out}(i)$ to be the $i$th column of $\widehat{U}$ and $\phi_{in}(i)$ to be the $i$th column of $\widehat{W}$, we get a distance embedding of $G$ into the unit sphere in $\R^{d+1}$:
$$ (i,j) \in E \ \Longleftrightarrow \ \phi_{out}(i) \cdot \phi_{in}(j) \geq 0 \ \Longleftrightarrow \ \|\phi_{out}(i) - \phi_{in}(j)\|^2 \leq 2.$$
\end{proof}

Combined with  Theorem~\ref{thm:distance-dimension-spherical}, this means that $|d_{dist}(G) - d_{sim}(G)| \leq 1$. In contrast, for undirected graphs, the minimum dimension needed by a dot-product embedding could be significantly less than for a distance embedding~\citep{RRS89}.

\subsection{Robust similarity embeddings}

Measuring the robustness of a similarity embedding is a bit different than for distance embeddings. For instance, the threshold for a similarity embedding need not even be positive, and thus a term of form $t(1+\delta)$ is not meaningful. We use an additive rather than multiplicative notion of robustness. 

\begin{defn}
We say a similarity embedding given by $(\phi_L, \phi_R, t)$ is $\delta$-robust, for $\delta > 0$, if
$$ (u,v) \not\in E \ \Longrightarrow \ \phi_L(u) \cdot \phi_R(v) \leq t - \delta \max_{w \in V } \max(\|\phi_L(w)\|^2, \|\phi_R(w)\|^2) .$$
\end{defn}
The term $\max_{w \in V }\|\phi(w)\|^2$ ensures that rescaling a similarity embedding does not change its robustness parameter. 

Theorem~\ref{thm:distance-embedding-existence} produces a distance embedding in the unit sphere, which is therefore also a similarity embedding. The following is an immediate corollary.
\begin{cor}
Let $G = (V,E)$ be any directed graph. If its adjacency matrix has rank $k$ and largest singular value $\sigma_1$, then $G$ has a $(1/\sigma_1)$-robust similarity embedding into the unit sphere in $\R^k$.
\label{cor:similarity-embedding-existence}
\end{cor}

\subsection{Relationship between similarity-robustness and distance-robustness}

We find that both presented definitions of robustness are closely linked. More specifically, robust similarity embeddings necessarily imply the existence of robust distance embeddings. Robust distance embeddings yield robust similarity embeddings only after normalizing by the diameter of an embedding which we define below.

\begin{defn}
The \emph{diameter} of a distance embedding $(\phi_{in}, \phi_{out}, t)$, denoted $\mbox{diam}(\phi)$, is the maximum distance between any two embedded vectors.  Define the \emph{diameter ratio}, $\mbox{dr}(\phi)$, to be $\mbox{diam}(\phi)/t$.
\end{defn}

We note that the diameter of most embeddings tends to be quite low ($O(1)$ for random graphs for example). 

\begin{thm}
Let $G$ be a directed graph with a $\delta$-robust distance embedding $(\phi_{in}, \phi_{out}, t)$ with diameter ratio $\mbox{dr}(\phi) = D$. Then $G$ has a similarity embedding with robustness $\Omega(\frac{\delta^2}{D^3})$ as $\frac{\delta}{D} \to 0$.  
\label{thm:delta-dist-sim}
\end{thm}

\begin{thm}
Let $G = (V,E)$ be a directed graph with a $\delta$-robust distance embedding $(\phi_{in}, \phi_{out}, t)$ into $\R^d$. Let $B$ be the largest length of the vectors $\phi_{in}, \phi_{out}$, and define the scaled diameter of the embedding as
$$ \Delta = \max \left(\frac{B}{t}, 1 \right) .$$
Then $G$ has a similarity embedding into $\R^{d+1}$ with robustness $\delta^2/(18 \Delta^4)$.
\label{thm:delta-dist-sim-2}
\end{thm}

We also find a relationship in the other direction.

\begin{thm} \label{thm:robust_sim_dist}
Let $G$ be a graph that has a $\delta$-robust similarity embedding. Then $G$ has distance-robustness at least $\delta/2$. 
\label{thm:delta-sim-dist}
\end{thm}

\subsection{Embeddings into the Hamming cube}

We now show that any graph that has a $\delta$-robust similarity embedding (into any dimension) can be embedded robustly into the $O((1/\delta^2) \log n)$-dimensional Hamming cube. Thus this notion of robustness translates directly into a bound on the number of {\it bits} of precision needed for embedding. 
\begin{thm}
Suppose directed graph $G = (V,E)$ has a $\delta$-robust similarity embedding into the unit sphere. Then it has an $O(\delta)$-robust distance embedding into $\{0,1\}^k$, where $k = O((1/\delta^2) \log n)$, that is simultaneously an $O(\delta)$-robust similarity embedding.
\label{thm:hamming}
\end{thm}
A proof is found in the appendix.

Notice that by combining Theorems~\ref{thm:distance-embedding-existence} and \ref{thm:hamming}, we see that any directed graph whose indegrees and outdegrees are bounded by $\Delta$ has both distance and similarity embeddings into the $O((1/\Delta^2) \log n)$-dimensional Hamming cube that are $O(1/\Delta)$-robust.

A partial converse is immediate: any distance or similarity embedding into $\{0,1\}^k$ is necessarily at least $(1/k)$-robust. Thus, robustness can serve as an approximate proxy for dimension. 

On the other hand, it is unclear whether embeddability in low-dimensional Euclidean space necessarily implies the existence of  a robust embedding. 
\begin{open}
Does the existence of a low-dimensional embedding imply that there also exists a robust embedding?
\end{open}

\section{Lower bounds}

\subsection{Sign rank}

Our previous construction from the proof of Theorem \ref{thm:relating-dimensions} with $M-tJ$ yields a matrix in which positive elements correspond to edges in $G$, and negative elements correspond to non-edges. This reveals a natural relationship between finding similarity embeddings and finding low rank sign matrices of an adjacency graph. 

\begin{defn}
Given a matrix of $+$s and $-$s, the sign rank of the matrix is said to the minimum rank of a matrix over the reals such that every entry agrees in sign with the corresponding $+$ or $-$. We use the convention that $0$ is neither $+$ nor $-$ and consequently the minimum rank matrix must have all non-zero elements.
\end{defn}

A matrix of $+$'s and $-$'s can be naturally interpreted as a directed graph, with $M_{ij} = +$ corresponding to an edge and $M_{ij} = -$ corresponding to a non-edge. 

\begin{defn}
The sign rank of a graph $G$ is the minimum sign rank of a sign matrix $M$ such that $M_{ij} = +$ if and only if $(i,j) \in E$. 
\end{defn}

Using the same construction we used in Theorem \ref{thm:relating-dimensions} we find that $d_{sign}(G)$ is closely linked to our other notions of dimension.

\begin{thm}
For any graph $G$, we have $d_{sim}(G) \leq d_{sign}(G) \leq d_{sim}(G) + 1$.
\label{thm:sign-rank}
\end{thm}

\subsection{Random graphs}

In this section, we show that random dense graphs have (with high probability) large embedding dimensions as well as low robust (with the former implying the latter). For convenience, we denote $d(G) = \min(d_{dist}(G), d_{sign}(G), d_{sim}(G)),$ and show that $d(G)$ is large for random graphs. We do this through a simple counting argument regarding the number of sign matrices of a given rank.

\begin{lemma}\label{lemma_alon}
(\cite{AMY16}) For $r \leq n/2$, the number of $n \times n$ sign matrices of sign rank at most $r$ does not exceed $2^{O(rn\log n)}$.
\end{lemma}

\begin{thm}
Let $G$ be a random directed graph over $n$ vertices such that each edge is chosen with constant probability $p$. Then as $n \to \infty$, with high probability, $$d(G) \geq O(\frac{nH(p)}{\log n}),$$ where $H(p) = -p\log p - (1-p)\log (1-p)$.
\end{thm}

\begin{proof}
Each $n \times n$ sign matrix is in direct correspondence with a directed graph $G$ , and $d_{sign}(G)$ is the sign rank of the matrix. 

Consider a random graph drawn by selecting each edge independently with probability $p$. Fix any $\epsilon > 0$ and consider the typical set (\cite{CT06}) induced by these random graphs (denoted $T_\epsilon$). It follows that for sufficiently large $n$,  
\begin{enumerate}
	\item With probability $1 - \epsilon$, our random graph $G \in T_\epsilon$
	\item For any $G \in T_\epsilon$, $P(G) \leq 2^{-n^2(H(p) - \epsilon)}$
\end{enumerate}
By Lemma \ref{lemma_alon}, for any $r$, the maximum number of elements in $T_\epsilon$ that have sign rank at most $r$ is $2^{O(rn\log n)}$, and consequently our random graph $G$ has rank at least $r$ with probability at least $$P(d(G) \geq r) \geq 1 - \epsilon - 2^{O(rn\log n)}2^{-(n^2)(H(p) - \epsilon)}.$$ Selecting $r = \frac{CnH(p)}{\log n}$ for a sufficiently small constant $C$ finishes the proof. 
\end{proof}

Applying Theorems \ref{thm_JL_project}, \ref{thm:delta-sim-dist} then shows that random graphs have low robustness.

\begin{cor}
Let $G$ be a random directed graph over $n$ vertices such that each edge is chosen with constant probability $p$. Then $G$ has distance robustness and similarity robustness at most $O(\frac{\log n}{\sqrt{n}})$.
\label{cor:random-graphs}
\end{cor}

\section{Algorithms}
\label{sec:algorithms}

We show in the appendix that computing $d_{dist}(G)$ and $d_{sim}(G)$ are both NP-hard problems. On the other hand, computing the robustness of a graph $G$ turns out to be far more tractable. 

We present a semidefinite programming approach to finding the distance-robustness and similarity-robustness of $G$. This can be used (see Theorems \ref{thm_JL_project} and \ref{thm:hamming}) to construct low dimensional robust embeddings of $G$. 

\subsection{Distance embeddings}

Given a graph $G$ with $V = \{v_1, v_2, \dots, v_n\}$, we find its distance robustness with a semidefinite program. For convenience, we will let $x_i$ denote $\phi_{out}(v_i)$ and $y_i$ denote $\phi_{in}(v_i)$. We also include a scalar variable $\delta$ which represents the robustness, and assume (without loss of generality) that our threshold $t = 1$. Then, the following semidefinite program suffices. 
\begin{equation*}
\begin{aligned}
 &\underset{x, y, t, \delta}{\text{maximize}}
& & \delta \\
& \text{subject to}
& & \langle x_i, x_i \rangle + \langle y_j, y_j \rangle - 2 \langle x_i, y_j \rangle \leq 1, \; (v_i, v_j) \in E \\
& & & \langle x_i, x_i \rangle + \langle y_j, y_j \rangle - 2 \langle x_i, y_j \rangle \geq 1 + \delta , \; (v_i, v_j) \notin E
\end{aligned}
\end{equation*}

\subsection{Similarity embeddings}

This similarity case is almost analogous, but has the detail that we restrict ourselves to unit vectors. This is still guaranteed to find the optimal robustness since any similarity embedding can be converted into a spherical embedding with the same robustness (albeit higher dimension).
\begin{equation*}
\begin{aligned}
 &\underset{x, y, t, \delta}{\text{maximize}}
& & \delta \\
& \text{subject to}
& &  \langle x_i, y_j \rangle \geq t, \; (v_i, v_j) \in E \\
& & & \langle x_i, y_j \rangle \leq t - \delta , \; (v_i, v_j) \notin E \\
& & & \langle x_i, x_i \rangle  = \langle y_i, y_i \rangle = 1, \; 1 \leq i \leq n
\end{aligned}
\end{equation*}

\section*{Acknowledgments}

We thank NSF under CNS 1804829 for research support.

\bibliography{refs}

\appendix

\section{Proofs to Selected Theorems}

\subsection{Proof of Corollary~\ref{cor:robustness-degree}}

\begin{proof}
This follows immediately from Theorem~\ref{thm:distance-embedding-existence} because the largest singular value of the adjacency matrix will be at most $\sqrt{\Delta_+\Delta_-}$. This is doubtless a well-known fact, but for completeness we give a brief explanation here.

The top singular value $\sigma_1$ is the square root of the top eigenvalue of $A^T A$, call it $\lambda$. Let $v$ be the corresponding eigenvector. Since $A^TA$ has no negative entries, we may assume $v \geq 0$ (flipping every entry of $v$ to its absolute value can only increase $v^T A^TA v$). If $v_i$ is the largest entry of $v$,
$$\lambda v_i = (A^T A v)_i = \sum_{j=1}^n (A^TA)_{ij} v_j \leq v_i \sum_j (A^TA)_{ij} = v_i \bigg(\sum_{\ell} A_{\ell i} \bigg(\sum_{j} A_{\ell j} \bigg)\bigg) \leq v_i \Delta_- \Delta_+ .$$ 
Thus $\lambda \leq \Delta_+\Delta_-$ and $\sigma_1 \leq \sqrt{\Delta_+\Delta_-}$.
\end{proof}

\subsection{Proof of Theorem~\ref{undir_diff_thm}}

\begin{proof}
It is known that $G$ has sphericity $O(n)$ (\cite{M84}). To embed $\dG$, let $A, B$ be its partitioning into independent sets. Then for any $a \in A$, $\phi_{out}(a) = -1, \phi_{in}(a) = 1$. For any $b \in B$, $\phi_{out}(b) = 1, \phi_{in}(b) = -1$. This embedding $\phi$, with $t=0$, is a distance embedding into $\R$.
\end{proof}

\subsection{Proof of Theorem~\ref{thm:delta-dist-sim}}

\begin{proof}
Recall our previous method (Theorem~\ref{thm:distance-dimension-spherical}) of placing distance embeddings on the unit sphere by mapping them onto a small neighborhood of the sphere. This method no longer suffices since picking too small a neighborhood would lead to very small robustness for the resulting similarity embedding. 

Conversely, trying to simply scale into a larger neighborhood can possibly distort distances enough to make the embedding no longer valid. As a result, we need to find an ``optimal'' neighborhood.

Let $e$ be a unit vector orthogonal to all vectors in $\phi$ (i.e. a new dimension). Let $\phi'$ be defined as the embedding such that $$\phi_{R}'(v) = \frac{e + \phi_{in}(v)}{\sqrt{1 + \|\phi_{in}(v)\|^2}},$$ $$\phi_{L}'(v) = \frac{e + \phi_{out}(v)}{\sqrt{1 + \|\phi_{out}(v)\|^2}}.$$ It follows that for any vertices $u,v$, we have $$\langle \phi'_{L}(u), \phi'_{R}(v) \rangle = \frac{2 + \|\phi_{out}(u)\|^2 + \|\phi_{in}(v)\|^2 - \|\phi_{out}(u) - \phi_{in}(v)\|^2}{2\sqrt{(1 + \|\phi_{out}(u)\|^2)(1 + \|\phi_{in}(v)\|^2)}}.$$
We now bound this quantity in the cases that $(u,v) \in E, (u,v) \notin E$. In doing so, we will show that $\phi'$ is a similarity embedding, into the unit sphere, of the desired robustness.

We will make repeated use of the following facts.
\begin{enumerate}
	\item $(u,v) \notin E$ if and only if $\|\phi_{out}(u) - \phi_{in}(v)\|^2 \geq (1+\delta)t^2$.
	\item $(u,v) \in E$ if and only if $\|\phi_{out}(u) - \phi_{in}(v)\|^2 \leq t$.
	\item Without loss of generality, let the origin be $\phi_{out}(u)$ for some arbitrary vertex $u$. Then all $\|\phi_{in}(u)\|$, $\|\phi_{out}(u)\|$ are $\leq Dt$.
\end{enumerate}

Suppose $(u,v) \in E$. Then 
\begin{equation*}
\begin{split}
\langle \phi'_{L}(u), \phi'_{R}(v) \rangle &= \frac{2 + \|\phi_{out}(u)\|^2 + \|\phi_{in}(v)\|^2 - \|\phi_{out}(u) - \phi_{in}(v)\|^2}{2\sqrt{(1 + \|\phi_{out}(u)\|^2)(1 + \|\phi_{in}(v)\|^2)}} \\
&\geq \frac{1}{2} \Big (\sqrt{\frac{1 + \|\phi_{out}(u)\|^2}{1 + \|\phi_{in}(v)\|^2}} + \sqrt{\frac{1 + \|\phi_{in}(v)\|^2}{1 + \|\phi_{out}(u)\|^2}} \Big ) - \frac{t^2}{2} \\
&\geq 1 - \frac{t^2}{2}.
\end{split}
\end{equation*}

Suppose $(u,v) \notin E$. Then 
\begin{equation*}
\begin{split}
\langle \phi'_{L}(u), \phi'_{R}(v) \rangle &= \frac{2 + \|\phi_{out}(u)\|^2 + \|\phi_{in}(v)\|^2 - \|\phi_{out}(u) - \phi_{in}(v)\|^2}{2\sqrt{(1 + \|\phi_{out}(u)\|^2)(1 + \|\phi_{in}(v)\|^2)}} \\
&\leq \frac{1}{2} \Big (\sqrt{\frac{1 + \|\phi_{out}(u)\|^2}{1 + \|\phi_{in}(v)\|^2}} + \sqrt{\frac{1 + \|\phi_{in}(v)\|^2}{1 + \|\phi_{out}(u)\|^2}} \Big ) - \frac{t^2(1+\delta)}{2(1 + D^2t^2)} \\
&\leq  \frac{1}{2} \Big (\sqrt{1 + D^2t^2} + \sqrt{\frac{1}{1 + D^2t^2}} \Big ) - \frac{t^2(1+\delta)}{2(1 + D^2t^2)} \\
&\leq 1 + \frac{3D^4t^4}{16} - \frac{t^2}{2}(1+\delta)(1 - D^2t^2)
\end{split}
\end{equation*}

Since $\phi'$ is clearly an embedding into the unit sphere, its similarity-robustness is simply the minimum difference between an edge ``dot product'' and a non-edge ``dot product''. Therefore, $\phi'$ must have similarity-robustness at least 
\begin{equation*}
\begin{split}
1 - \frac{t^2}{2} - \Big (1 + \frac{3D^4t^4}{16} - \frac{t^2}{2}(1+\delta)(1 - D^2t^2) \Big ) \\ = \frac{t^2}{2}(\delta  - D^2t^2 - \delta D^2t^2) - \frac{3D^4t^4}{16}.
\end{split}
\end{equation*}
The key idea is that we can scale $\phi$ as we like, which means that we can select $t$ to be any value we choose. Thus, selecting $t = O(\sqrt{\frac{\delta}{D^3}})$, we have,
\begin{equation*}
\begin{split}
&= \frac{O(\delta)}{2D^3}(\delta - O(\frac{\delta}{D}) - O(\frac{\delta^2}{D})) - O(\frac{\delta^2}{D^2}) \\
&= O(\frac{\delta^2}{D^3}), \text{ as }\frac{\delta}{D} \to 0.
\end{split}
\end{equation*}

\end{proof}

\subsection{Proof of Theorem~\ref{thm:delta-dist-sim-2}}

\begin{proof}
Recall our earlier idea, in the proof sketch for Theorem~\ref{thm:distance-dimension-spherical}, of placing distance embeddings on the unit sphere by mapping them onto a small neighborhood of the sphere. We will now look at a particular realization of this method.

The distance embedding given by $\phi_{in}, \phi_{out}: V \rightarrow \R^d$ can be scaled so that $t=1$. We then get the following, for all $u,v \in V$.
\begin{enumerate}
	\item[(a)] $(u,v) \notin E$ if and only if $\|\phi_{out}(u) - \phi_{in}(v)\|^2 \geq 1+\delta$.
	\item[(b)] $(u,v) \in E$ if and only if $\|\phi_{out}(u) - \phi_{in}(v)\|^2 \leq 1$.
	\item[(c)] All $\|\phi_{in}(u)\|$, $\|\phi_{out}(u)\|$ are $\leq \Delta$.
\end{enumerate}

Let $e$ be a unit vector orthogonal to all embedded vectors (i.e. a new dimension). For some constant $c > 0$ whose value we will later set, let $\phi_L', \phi_R': V \rightarrow \R^{d+1}$ be defined by \begin{align*}
\phi_{R}'(v) &= \frac{e + c \, \phi_{in}(v)}{\sqrt{1 + c^2\|\phi_{in}(v)\|^2}} \\
\phi_{L}'(v) &= \frac{e + c \, \phi_{out}(v)}{\sqrt{1 + c^2 \|\phi_{out}(v)\|^2}}
\end{align*}
Notice that these vectors have unit length. It follows that for any $u,v$, 
\begin{align*}
\langle \phi'_{L}(u), \phi'_{R}(v) \rangle 
&= \frac{1 + c^2 \langle \phi_{out}(u), \phi_{in}(v) \rangle}{\sqrt{(1 + c^2 \|\phi_{out}(u)\|^2)(1 + c^2 \|\phi_{in}(v)\|^2)}} \\
&= \frac{(1 + c^2 \|\phi_{out}(u)\|^2) + (1 + c^2 \|\phi_{in}(v)\|^2) - c^2 \|\phi_{out}(u) - \phi_{in}(v)\|^2}{2\sqrt{(1 + c^2 \|\phi_{out}(u)\|^2)(1 + c^2 \|\phi_{in}(v)\|^2)}} 
\end{align*}

We now bound this quantity in the cases that $(u,v) \in E$ and $(u,v) \notin E$. In doing so, we will show that $\phi'$ is a similarity embedding, into the unit sphere, of the desired robustness.

Suppose $(u,v) \in E$. Using the inequality $A + B \geq 2 \sqrt{AB}$ for $A,B \geq 0$ as well as property (b), we have
\begin{align*}
\langle \phi'_{L}(u), \phi'_{R}(v) \rangle 
&\geq 1 - \frac{c^2}{2} \cdot \frac{\|\phi_{out}(u) - \phi_{in}(v)\|^2}{\sqrt{(1 + c^2 \|\phi_{out}(u)\|^2)(1 + c^2 \|\phi_{in}(v)\|^2)}} \\
&\geq 1 - \frac{c^2}{2}.
\end{align*}

On the other hand, if $(u,v) \notin E$, then by properties (a) and (c),
\begin{align*}
\langle \phi'_{L}(u), \phi'_{R}(v) \rangle 
&\leq \frac{1}{2} \left(\sqrt{\frac{1 + c^2 \|\phi_{out}(u)\|^2}{1 + c^2 \|\phi_{in}(v)\|^2}} + \sqrt{\frac{1 + c^2\|\phi_{in}(v)\|^2}{1 + c^2\|\phi_{out}(u)\|^2}} \right) - \frac{c^2}{2} \cdot \frac{1+\delta}{1 + c^2 \Delta^2} .
\end{align*}
We can simplify the first term using the inequalities $(1+x) + 1/(1+x) \leq 2 + 2x^2$ and $\sqrt{1+x} \leq 1 + x/2$ for $x \geq 0$. Again using property (c), we get
\begin{align*}
\langle \phi'_{L}(u), \phi'_{R}(v) \rangle 
&\leq 1 + \frac{c^4 \Delta^4}{4} - \frac{c^2}{2} \cdot \frac{1+\delta}{1 + c^2 \Delta^2} .
\end{align*}
Set $c = \sqrt{\delta/(3 \Delta^4)}$. Then $c^2 \Delta^2 \leq \delta/3$ and we get
$$
\langle \phi'_{L}(u), \phi'_{R}(v) \rangle 
\ \leq \ 
1 - \frac{c^2}{2} \left( 1 + \frac{\delta}{2} - \frac{c^2 \Delta^4}{2} \right) 
\ = \ 
1 - \frac{c^2}{2} \left( 1 + \frac{\delta}{3} \right) .
$$
The robustness of a similarity embedding is measured additively, and follows by taking the difference of the expressions for the cases when $(u,v) \in E$ and $(u,v) \not\in E$.
\end{proof}

\subsection{Proof of Theorem~\ref{thm:robust_sim_dist}}

\begin{proof}
 Let $\phi$ be a $\delta$-robust similarity embedding of $G$. Rescale the embedding so that $\phi_{L,R}(v)$ all have norm at most $1$. Then it follows that for some $t \in [-1,1]$,
\begin{enumerate}
	\item $(u,v) \in E$ if and only if $\langle \phi_{L}(u), \phi_{R}(v) \rangle \geq t$.
	\item $(u,v) \notin E$, if and only if $\langle \phi_{L}(u), \phi_{R}(v) \rangle \leq t - \delta$.
\end{enumerate}
Next, we convert this embedding into a spherical embedding as follows. Let $e, f$ be unit vectors that are orthogonal to each other and to all vectors in our embedding. We append suitable multiples of $e$ to each $\phi_{L}(v)$ vector and of $f$ to each $\phi_{R}(v)$, so that the resulting vectors all have unit length. This operation preserves dot products and thus gives a spherical embedding $\phi'$ with robustness $\delta$.

Since $\phi'$ is spherical, it is also a distance embedding with $\phi'_{L,R} = \phi'_{out, in}$ where
\begin{enumerate}
	\item $(u,v) \in E$ if and only if $\|\phi'_{out}(u) - \phi'_{in}(v)\|^2 \leq 2 - 2t$.
	\item $(u,v) \notin E$ if and only if $\|\phi'_{out}(u) - \phi'_{in}(v)\|^2 \geq 2 - 2t + 2\delta$.
\end{enumerate}
From here, we can lower-bound the distance robustness of $\phi'$ by
$$ \frac{2 - 2t + 2\delta}{2 - 2t} - 1
\ \geq \ \frac{\delta}{2}.
$$
\end{proof}

\subsection{Proof of Theorem~\ref{thm:hamming}}

\begin{proof}
Write $n = |V|$. Suppose the $2n$ vectors $\phi_L(u)$, $\phi_R(u)$ lie on $S^{d-1}$, the unit sphere in $\R^d$, and constitute a $\delta$-robust similarity embedding: for some threshold $t$, and any $u, v$,
\begin{itemize}
\item $(u,v) \in E \implies \phi_L(u) \cdot \phi_R(v) \geq t + \delta$, and
\item $(u,v) \not\in E \implies \phi_L(u) \cdot \phi_R(v) \leq t$.
\end{itemize}
We will embed these vectors into the Hamming cube using the random halfspace method of \cite{GW95} and \cite{C02}. Specifically, pick $k$ vectors $r_1, \ldots, r_k$ uniformly at random from $S^{d-1}$, and define the embedding $h: \R^d \rightarrow \{0,1\}^k$ by $h(x) = (h_1(x), \ldots, h_k(x))$, where the $i$th hash function $h_i: \R^d \rightarrow \{0,1\}$ is
$$ h_i(x) 
= 
\left\{
\begin{array}{ll}
1 & \mbox{if $r_i \cdot x \geq 0$} \\
0 & \mbox{if $r_i \cdot x < 0$}
\end{array}
\right.
$$
Now, for any vectors $x,y$,
$$ \pr(h_i(x) \neq h_i(y)) = \pr((r_i \cdot x)(r_i \cdot y) \leq 0) = \frac{\theta}{\pi},$$
where $\theta$ is the angle between $x$ and $y$. Thus for nodes $u,v$ in $G$,
\begin{align}
\pr(h_i(\phi_L(u)) \neq h_i(\phi_R(v))) &= \arccos(\phi_L(u) \cdot \phi_R(v))/\pi \notag \\
&\left\{
\begin{array}{ll}
\leq \arccos(t+\delta)/\pi & \mbox{if $(u,v) \in E$} \\
\geq \arccos(t)/\pi & \mbox{if $(u,v) \not\in E$}
\end{array}
\right.
\label{eq:expected-hamming-dist}
\end{align}
The difference between the two options is:
$$ \frac{1}{\pi}\left( \arccos(t) - \arccos(t+\delta) \right) = \frac{1}{\pi} \int_{t}^{t+\delta} \frac{dz}{\sqrt{1-z^2}} \geq \frac{\delta}{\pi} .$$
Write $h_L(u) = h(\phi_L(u))$ and $h_R(u) = h(\phi_R(u))$. Letting $d(\cdot)$ denote Hamming distance in $\{0,1\}^k$, we have that the expected value of $d(h_L(u), h_R(v))$ is $k$ times the quantity in equation (\ref{eq:expected-hamming-dist}). A simple Chernoff-Hoeffding bound, unioned over all pairs $u,v$, then suffices to show that  if $k = O((1/\delta^2) \log n)$, then with probability at least $1-1/n$,
\begin{itemize}
\item $(u,v) \in E \implies d(h_L(u), h_R(v)) \leq k (\arccos(t) - 2\delta/3)/\pi$, and
\item $(u,v) \not\in E \implies d(h_L(u), h_R(v)) \geq k (\arccos(t) - \delta/3)/\pi$.
\end{itemize}
Thus $h_L, h_R$ constitute an $O(\delta)$-robust distance embedding. To see that this is also an $O(\delta)$-robust similarity embedding, notice that all the embedded vectors $h_L(u)$ and $h_R(u)$ have expected squared Euclidean norm $k/2$, and given the setting of $k$, these norms will be tightly concentrated, within multiplicative factor $1 \pm O(\delta)$, about their expected values.
\end{proof}

\section{NP hardness results}

In this section, we show that it is NP-hard to find distance or similarity embeddings of minimum dimension. We do so by adapting the results of \citet{KM12} on undirected embeddings to the directed case. First, we briefly review some definitions from their paper. Readers interested in further details should consult their very clear presentation. 

\begin{defn}
(\cite{KM12}) 
\begin{enumerate}
	\item An \emph{oriented $k$-hyperplane arrangement} $\mathcal{H} = \{h_1, h_2, \dots, h_n\}$ is a set of hyperplanes in $\R^k$ each of which is given an orientation, so that all points in $\R^k$ are either on the positive side of $h_i$, denoted $h_i^+$, or the negative side, denoted $h_i^-$, or on $h_i$ itself. 
	\item The \emph{sign vector} of a point $p \in \R^k$ with respect to $\mathcal{H}$ is the vector $\sigma(p) \in \{+,-,0\}^n$ such that
	$$ \sigma(p)_i
	= 
	\left\{
	\begin{array}{ll}
	+ & \mbox{if $p \in h_i^+$} \\
	- & \mbox{if $p \in h_i^-$} \\
	0 & \mbox{if $p \in h_i$}
	\end{array}
	\right.
	$$
	\item The \emph{combinatorial description} of $\mathcal{H}$ is defined to be the set of all sign vectors, $\D(\H) = \{\sigma(p): p \in \R^k\}$.
	\item Consider any set $S \subset \{-, +\}^n$ containing $(-,-,-,\dots,-)$ and $(+,+,\dots,+)$. We say $S$ is \emph{$k$-realizable} if there exists an oriented $k$-hyperplane arrangement $\mathcal{H}$ with $S \subset \D(\H)$. 
	\item $k$-{\sc realizability} denotes the algorithmic problem of deciding, given a set $S \subset \{-, +\}^n$ as input, whether $S$ is $k$-realizable. 
\end{enumerate}
\end{defn}

\begin{thm}
(\cite{KM12}) $k$-{\sc realizability} is NP-hard for all $k > 1$.
\end{thm}

\subsection{Distance embeddings}

The main idea is to reduce from $k$-{\sc realizability}. Given a set $S \subset \{-, +\}^n$, we will construct a graph $G(S)$ in polynomial time,  such that $S$ is $k$-realizable if and only if $G(S)$ has a $k$-dimensional distance embedding. For convenience, we start by presenting the construction from \citet{KM12}. 

\begin{defn}
(\cite{KM12}) For any $S \subset \{-, +\}^n$, define $G_U(S) = (V,E)$ to be the undirected graph with vertices
$$V = \{a_1, a_2, \dots a_n\} \cup  \{b_1, b_2, \dots b_n\} \cup  \{c_\sigma: \sigma \in S \}$$
and edges
\begin{itemize}
	\item $\{c_\sigma, c_\pi\} \in E$ for all $\sigma, \pi \in S$
	\item $\{a_i, a_j\}, \{b_i, b_j\} \in E$ for all $i \neq j$
	\item $\{a_i, c_{\sigma}\} \in E$ if and only if $\sigma_i = +$
	\item $\{b_i, c_\sigma\} \in E$ if and only if $\sigma_i = -$.
\end{itemize}
\end{defn}

\begin{thm}\label{thm_km_dist}
(\cite{KM12}) $S$ is $k$-realizable if and only if $G_U(S)$ has a $k$-dimensional undirected distance embedding. 
\end{thm}

Our directed construction is very similar to $G_U(S)$.

\begin{defn}
For $S \subset \{-, +\}^n$, let $G_D(S) = (V,E)$ be the directed graph with vertices
$$V = \{a_1, a_2, \dots a_n\} \cup  \{b_1, b_2, \dots b_n\} \cup  \{c_\sigma, \sigma \in S\}$$
and edges
\begin{itemize}
	\item $(a_i, c_{\sigma}) \in E$ if and only if $\sigma_i = +$
	\item $(b_i, c_\sigma) \in E$ if and only if $\sigma_i = -$.
\end{itemize}
\end{defn}

\begin{thm}
$S$ is $k$-realizable if and only if $G_D(S)$ has a distance embedding of dimension at most $k$.
\end{thm}

\begin{proof}
\begin{enumerate}
	\item[$\Rightarrow$] Suppose $G_D(S)$ has a distance embedding $\phi$ with dimension $k$. Let $h_i$ be the hyperplane that is the perpendicular bisector of $\phi_{out}(a_i)$ and $\phi_{out}(b_i)$; orient it so that $\phi_{out}(a_i)$ lies on the positive side and $\phi_{out}(b_i)$ on the negative side. Letting $\H = \{h_1, h_2, \dots, h_n\}$, we claim that $\phi_{in}(c_{\sigma})$ has sign vector exactly $\sigma$ with respect to $\H$. Thus $S \subset \D(\H)$.
	
To prove our claim, consider any $\sigma \in S$, $1 \leq i \leq n$. $c_{\sigma}$ has an edge from exactly one of $a_i$ and $b_i$ and consequently $\phi_{in}(c_{\sigma})$ is closer to the corresponding $\phi_{out}(a_i)$ or $\phi_{out}(b_i)$. Thus $\phi_{in}(c_{\sigma})$ is on the $\sigma_i$ side of $h_i$. Combining this over all $i$, we see that $\phi_{in}(c_\sigma)$ has sign vector $\sigma$ as desired. 
	\item[$\Leftarrow$] Suppose $S$ is $k$-realizable. Then by Theorem \ref{thm_km_dist}, $G_U(S)$ has a $k$-dimensional undirected distance embedding $\phi$ (in the sense of Theorem~\ref{thm:sphericity}). We construct a $k$-dimensional distance embedding $\phi'$ of $G_D(S)$ from $\phi$ as follows. $\phi'_{out}(a_i) = \phi(a_i)$, $\phi'_{out}(b_i) = \phi(b_i)$, $\phi'_{in}(c_\sigma) = \phi(c_\sigma)$, and the remaining vectors are assigned so that they are each at distance $> 1$ from all the other vectors. The only possible edges in such an embedding are edges from $\{a_1, a_2, \dots a_n, b_1, b_2, \dots b_n\} \to \{c_\sigma: \sigma \in S\}$. Since the corresponding edges are in $G_U(S)$, it follows that $\phi'$ is a valid embedding of $G_D(S)$ as desired. 
\end{enumerate}

\end{proof}

Since the construction of $G_D(S)$ from $S$ takes polynomial time, the hardness of $k$-{\sc realizability} implies the following.
\begin{cor}
Computing $d_{dist}(G)$ for a directed graph $G$ is NP-hard.
\end{cor}

\subsection{Similarity embeddings}

This section is almost identical to the previous section. The only difference is that our constructions $G_D(S)$ and $G_U(S)$ are different to account for the fact that we are dealing with similarity embeddings instead of distance embeddings. 

\begin{defn}
(\cite{KM12}) For $S \subset \{-, +\}^n$, $G_U(S) = (V,E)$ is the undirected graph defined as follows. $$V = \{a_1, a_2, \dots a_n\} \cup  \{c_\sigma: \sigma \in S \}.$$ $E$ is defined by
\begin{itemize}
	\item $\{a_i, a_j\} \in E$ for all $i \neq j$
	\item $\{a_i, c_{\sigma}\} \in E$ if and only if $\sigma_i = +$.
\end{itemize}
\end{defn}

\begin{thm}\label{thm_km_sim}
(\cite{KM12}) $S$ is $k$-realizable if and only if $G_U(S)$ has a $k$-dimensional undirected similarity embedding (in the sense of Theorem~\ref{thm:sphericity-dot-product}). 
\end{thm}

\begin{defn}
 For $S \subset \{-, +\}^n$, $G_D(S) = (V,E)$ is the directed graph defined as follows. $$V = \{a_1, a_2, \dots a_n\} \cup  \{c_{\sigma}: \sigma \in S\}.$$ $$E = \{(a_i, c_\sigma): \sigma_i = +\}.$$
\end{defn}

\begin{thm}
$S$ is $k$-realizable if and only if $d_{sim}(G_D(S)) \leq k$. 
\end{thm}

\begin{proof}
\begin{enumerate}
	\item[$\Rightarrow$] Suppose $G_D(S)$ has a similarity embedding $\phi$ with dimension $k$. Let $$h_i = \{v: \langle v, \phi_{L}(a_i) \rangle = t\},$$ and let $\H = \{h_1, h_2, \dots, h_n\}$. We claim that $\phi_{R}(c_{\sigma})$ has sign vector exactly $\sigma$ with respect to $\H$. This clearly suffices as it shows that $\sigma \in \D(\H)$ for all $\sigma \in S$. The proof is analogous to case presented in Theorem \ref{thm_km_dist}.
	\item[$\Leftarrow$] Suppose $S$ is $k$-realizable. Then by Theorem \ref{thm_km_dist}, $G_U(S)$ has a $k$ dimensional undirected similarity embedding $\phi$ with threshold $t = 1$. Our construction of a similarity embedding for $G_D(S)$ is identical to our construction in Theorem \ref{thm_km_dist} with the only difference being that our remaining points are mapped to $0$ instead of infinity. Since our threshold $t > 0$, this means that none of the non edges are embedded, and this completes the proof. 
\end{enumerate}

\end{proof}

\begin{cor}
Computing $d_{sim}(G)$ for a directed graph $G$ is NP-hard.
\end{cor}

\end{document}